\documentclass[11pt,reqno]{amsart}
\usepackage[margin=1.0in, top=1in]{geometry}
\usepackage{amssymb, mathtools}
\usepackage{booktabs, array, float}
\usepackage{graphicx}
\usepackage{algorithm, algpseudocode}
\usepackage[most]{tcolorbox}
\usepackage{tikz}           
\usepackage{xcolor}         
\usepackage{orcidlink}      
\definecolor{navyblue}{RGB}{10, 10, 50} 
\definecolor{darkgray}{RGB}{60, 60, 60} 
\newcommand\blfootnote[1]{%
  \begingroup
  \renewcommand\thefootnote{}\footnote{#1}%
  \addtocounter{footnote}{-1}%
  \endgroup
}

\theoremstyle{plain}
\newtheorem{theorem}{Theorem}[section]

\theoremstyle{definition}

\theoremstyle{remark}
\newtheorem{remark}[theorem]{Remark}

\usepackage{hyperref}
\hypersetup{colorlinks=true, linkcolor=blue, citecolor=blue, urlcolor=blue}
\usepackage{pgfplots}
\pgfplotsset{compat=1.18}
\usetikzlibrary{arrows.meta}

\begin{document}

\begin{tikzpicture}[remember picture,overlay]
  \fill[navyblue] (current page.north west) rectangle ([yshift=-1cm]current page.north east);
  \fill[darkgray] ([yshift=-1cm]current page.north west) rectangle ([yshift=-1.05cm]current page.north east);
  \node[anchor=north west, xshift=0.5cm, yshift=-0.2cm] at (current page.north west) {\includegraphics[height=2.2cm]{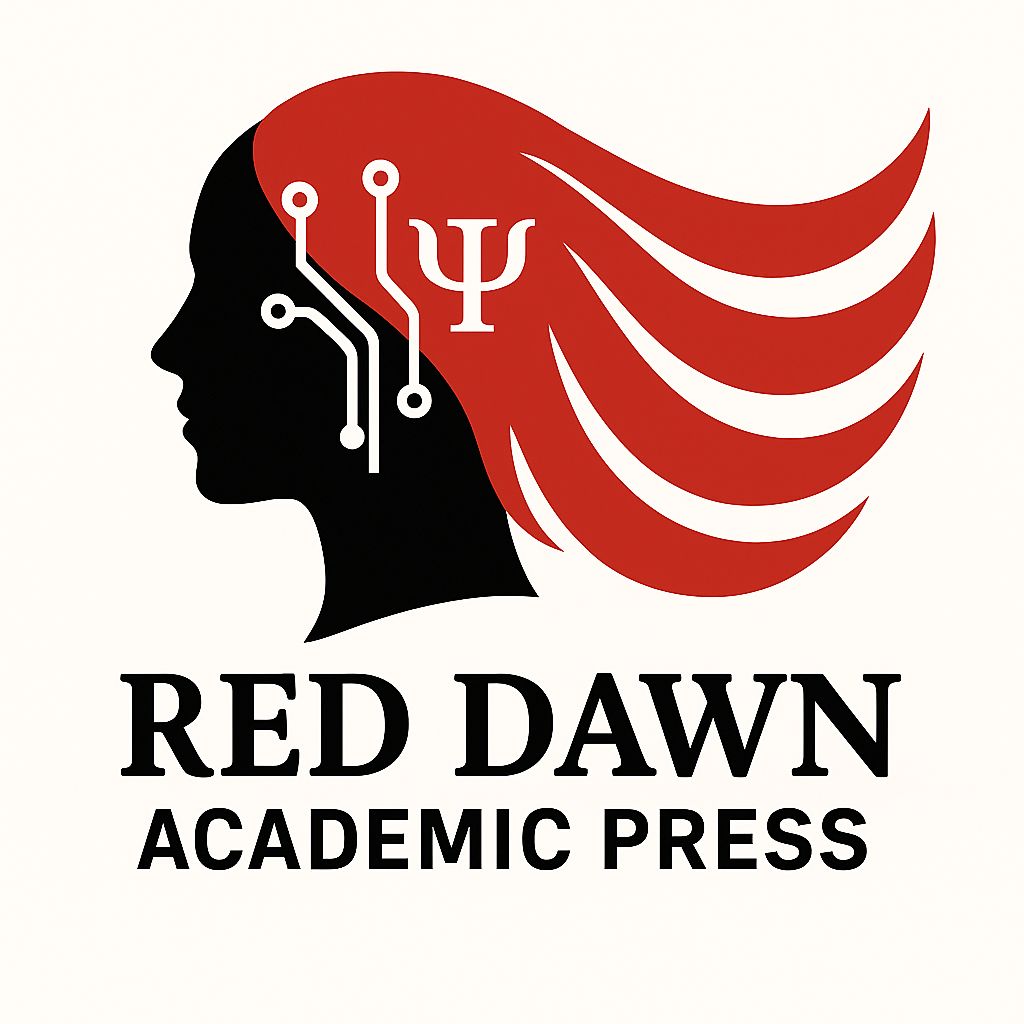}};
\end{tikzpicture}
\begin{center}
    \vspace*{1.5cm} 

    {\Large\bfseries Temporal Lifting as Latent-Space Regularization for Continuous-Time Flow Models in AI Systems}

    \vspace{1.0cm}
    {\large Jeffrey Camlin} 
    \vspace{1.5cm}
\end{center}
\blfootnote{\textbf{Affiliation:} \href{https://reddawnacademicpress.org/}{Red Dawn Academic Press}}
\blfootnote{\textbf{ORCID:} \orcidlink{0000-0002-5740-4204} \href{https://orcid.org/0000-0002-5740-4204}{0000-0002-5740-4204}}
\blfootnote{\textbf{Date:} Preprint, 8 October 2025 (Revision 2: 26 January 2026)}

\vspace{1cm}
\begin{center}
    \small
    \begin{minipage}{0.85\textwidth}
        \textbf{Abstract.} We present a latent-space formulation of adaptive temporal reparametrization for continuous-time dynamical systems. The method, called \emph{temporal lifting}, introduces a smooth monotone mapping $t \mapsto \tau(t)$ that regularizes near-singular behavior of the underlying flow while preserving its conservation laws. In the lifted coordinate, trajectories such as those of the incompressible Navier--Stokes equations on the torus $\mathbb{T}^3$ become globally smooth. From the standpoint of machine-learning dynamics, temporal lifting acts as a continuous-time normalization or time-warping operator that can stabilize physics-informed neural networks and other latent-flow architectures used in AI systems. The framework links analytic regularity theory with representation-learning methods for stiff or turbulent processes.

        \vspace{0.4cm}

        \noindent\textbf{ACM 2012 classification:} I.2.0. PACS codes: 47.10.ad; 47.27.eb; 02.30.Jr. MSC 2020 codes: 35Q30; 76D05; 65M70; 68T07; 68T27; 03D45.

        \vspace{0.2cm}

    \end{minipage}
\end{center}



\vspace{5mm}

\tableofcontents

\newpage

\section{Introduction}
\label{sec1}

The incompressible Navier--Stokes equations on the three--torus 
$\mathbb{T}^3 = \mathbb{R}^3 / \mathbb{Z}^3$ exhibit a fundamental tension 
between the periodicity of the spatial domain and the unbounded character 
of the temporal axis. Classical approaches treat physical time $t \in [0,\infty)$ as an external parameter as a neutral clock that labels solution states without participating in the analytic structure of the problem.

This perspective underlies the standard view that coordinate changes of the form $\tilde{t} = \varphi(t)$ constitute mere gauge transformations: the solution trajectory is relabeled, but its regularity class, energy bounds, and blowup criteria remain unaffected. Time, in this formulation, is analytically inert.

However, on a periodic domain, this separation is artificial. The lattice 
structure $\mathbb{Z}^3$ that defines $\mathbb{T}^3$ induces geometric 
constraints of symmetry axes, fundamental domains, spectral localization that have no analogue in the temporal direction under the classical formulation. The question arises: can the temporal coordinate be embedded into the geometric structure of the torus in a way that is analytically nontrivial?

We answer this in the affirmative. The method, which we call \emph{temporal lifting}, constructs a smooth monotone mapping $t = \varphi(\tau)$ that aligns the computational trajectory with the intrinsic geometry of the periodic domain. Unlike classical coordinate changes, temporal lifting exploits the lattice symmetries of $\mathbb{T}^3$ to achieve regularization effects that are not available on generic (non-periodic) domains.

\subsection{Temporal Lifting as Latent-Space Regularization}
\label{subsec1}

From the perspective of representation learning, the ``analytically inert'' nature of physical time $t$ introduces a severe conditioning problem. In continuous-time flow models, such as Neural ODEs~\cite{chen2018} and Physics-Informed Neural Networks (PINNs)~\cite{raissi2019,karniadakis2021}, stiff dynamics near singular regions manifest as gradient explosion. The network is forced to learn a mapping where the Lipschitz constant of the vector field approaches infinity.

We introduce a dual-domain computational architecture where the continuous-time trajectory is integrated in a regularized latent space $\tau$. By maintaining the differential relation $dt/d\tau = \varphi'(\tau)$, the system isolates the numerical integrator from the stiffness of the physical coordinate $t$. This ensures that even as the physical dynamics approach high-gradient regimes, the computational domain $\tau$ remains well-conditioned, allowing stable integration through regions where conventional methods fail.

This approach provides three key properties: (i) a geometric inductive bias 
that anchors the computational coordinate $\tau$ to the lattice structure of 
$\mathbb{T}^3$; (ii) a continuous-time normalization via the dilation factor 
$\varphi'(\tau)$, which unfolds high-gradient regimes into smooth latent 
trajectories; and (iii) stability through spectral localization rather than 
manual adaptive stepping.

The construction is motivated by the Path Lifting Lemma in covering space 
theory~\cite{hatcher2002}: a loop on $S^1$ that appears discontinuous can be 
lifted to a smooth path on the universal cover $\mathbb{R}$. Temporal lifting 
applies this principle to the Fourier-Galerkin trajectories on the torus, 
where the periodic structure provides the covering geometry.

\subsection{Temporal lifting and motivation}
\label{subsec1}

In contrast, we adopt the term \emph{temporal lifting} to describe a constructive analytic procedure:  
\[
\tilde{t} = \varphi(t), 
\qquad 
\tilde{u}(x,\tilde{t}) = u(x,\varphi^{-1}(\tilde{t})),
\]
where $\varphi \in C^\infty$, $\varphi' > 0$, is chosen adaptively to smooth derivative discontinuities 
at singular times. Unlike mere reparametrization, temporal lifting has tangible analytic consequences: 
a trajectory that is only piecewise smooth in $t$ may become globally $C^\infty$ in $\tilde{t}$. 
This device is motivated by the geometric analogy of the Path Lifting Lemma in covering space theory~\cite{hatcher2002}, 
where a loop on the circle $S^1$ can be lifted to a smooth path on the universal cover 
$\mathbb{R}$, removing apparent discontinuities.

\begin{figure}[H]
\centering
\includegraphics[width=0.75\textwidth]{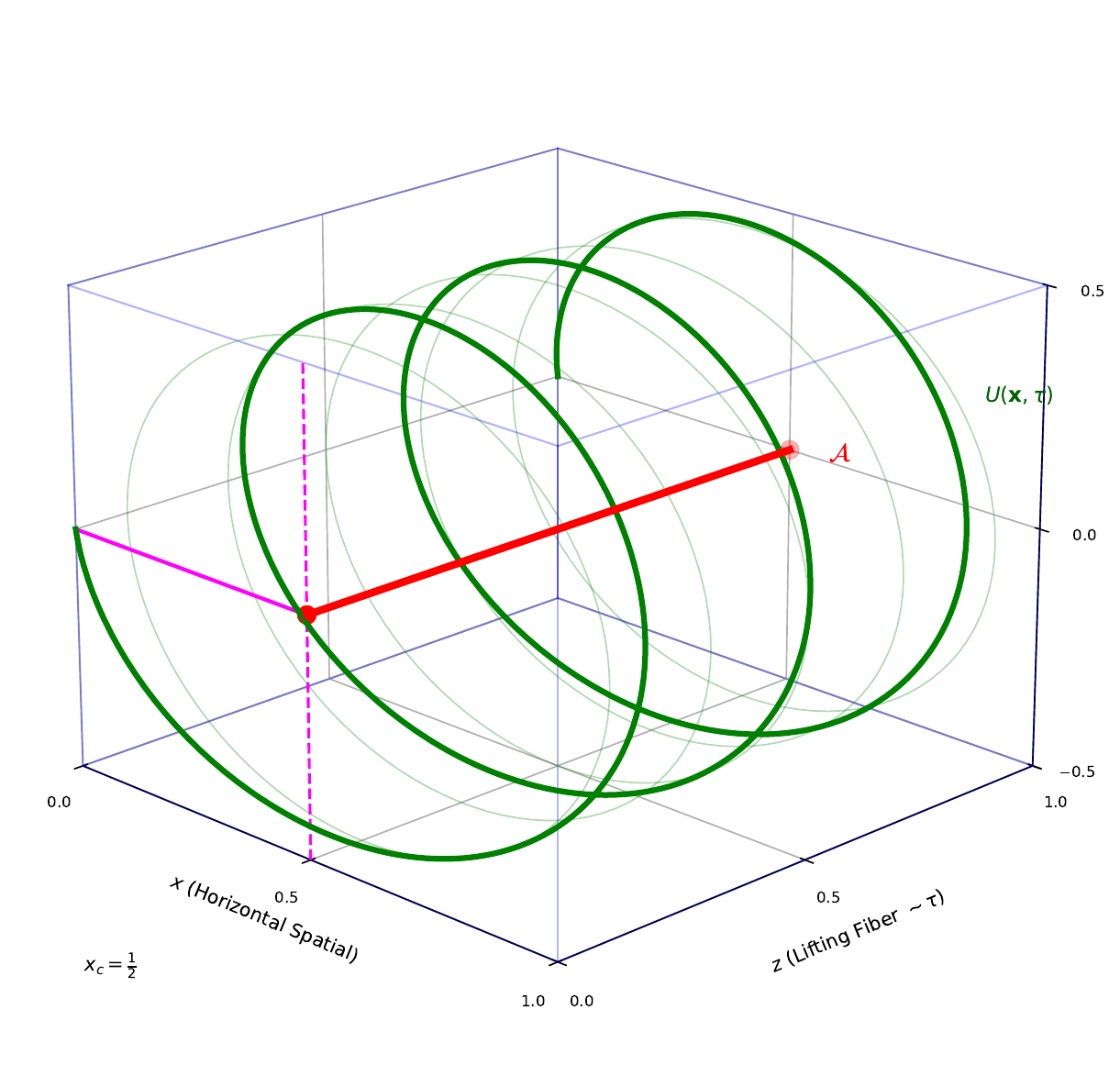}
\caption{\small Dual-axis latent geometry of temporal lifting.
\textbf{Blue:} Bounded computational domain (fundamental cell of $\mathbb{T}^3$). 
Periodicity ensures exact Fourier-Galerkin discretization with no boundary artifacts.
\textbf{Red:} Central axis $\mathcal{A}$ at $x_c = 1/2$ — spatial normalization 
anchor established by the discrete lattice symmetry $\mathbb{Z}^3$.
\textbf{Green:} Lifted trajectory $U(\mathbf{x}, \tau)$ evolving smoothly 
through the latent space; compactness guarantees bounded dynamics.
\textbf{Magenta:} Spectral projection onto the periodic basis — temporal 
normalization via $\varphi(\tau)$.
Unlike standard normalizing flows, the method regularizes along both spatial 
and temporal axes simultaneously, exploiting the lattice structure of 
$\mathbb{T}^3$ for enhanced conditioning.}
\label{fig:latent-geometry}
\end{figure}

\section{Preliminaries}
\label{sec2}
This section establishes the theoretical foundation for the temporal lifting framework.
\subsection{Function Spaces and Navier--Stokes Equations}
\label{subsec2.1}

Let $\mathbb{T}^3 := \mathbb{R}^3 / \mathbb{Z}^3$ denote the three--torus. 
We use standard Lebesgue spaces $L^p(\mathbb{T}^3)$ and Sobolev spaces $H^s(\mathbb{T}^3)$ 
for $s \ge 0$~\cite{adams2003,evans2010}. 
The divergence--free subspace is defined by
\begin{equation}
H^s_{\mathrm{div}}(\mathbb{T}^3) :=
\{\, u \in H^s(\mathbb{T}^3)^3 : \nabla \cdot u = 0 \,\}.
\label{eq:divspace}
\end{equation}
We write $\|\cdot\|_{H^s}$ for the $H^s$ norm and $\|\cdot\|_{L^2}$ for the $L^2$ norm.

The incompressible Navier--Stokes equations on $\mathbb{T}^3$ are given by
\begin{align}
\partial_t u + (u \cdot \nabla)u + \nabla p - \nu \Delta u &= 0, \label{eq:NSE}\\
\nabla \cdot u &= 0, \label{eq:div}
\end{align}
for velocity $u(x,t) \in \mathbb{R}^3$, pressure $p(x,t) \in \mathbb{R}$, viscosity $\nu > 0$, 
and initial data
\begin{equation}
u(x,0) = u_0(x) \in H^s_{\mathrm{div}}(\mathbb{T}^3),
\label{eq:init}
\end{equation}
with $s$ sufficiently large. 
We follow the classical framework of Leray~\cite{leray1934} and Hopf~\cite{hopf1951}.

\subsection{Temporal Lifting}
\label{subsec2.2}

Let $\varphi \in C^\infty([0,\infty))$ with $\varphi' > 0$. 
Define the \emph{lifted trajectory} by
\begin{equation}
U(x,\tau) := u(x,\varphi(\tau)), 
\qquad 
t = \varphi(\tau).
\label{eq:liftdef}
\end{equation}
We call this procedure \emph{temporal lifting}. 
Unlike classical time reparametrization---a neutral coordinate change---temporal lifting is chosen 
adaptively to smooth derivative discontinuities at singular times and restore global 
$C^\infty$ regularity. On the periodic domain $\mathbb{T}^3 = \mathbb{R}^3/\mathbb{Z}^3$, 
the lifting geometry is not arbitrary: natural choices align the computational trajectory 
with the centroid of the fundamental domain, exploiting lattice symmetries for enhanced stability.

\vspace{5mm}
\section{Main Theorem}
\label{sec3}

\begin{theorem}[Temporal Lift Equivalence Theorem]
\label{thm:temporal-lift}

Let $u(x,t)$ be a Leray--Hopf (resp.\ classical) solution of the incompressible 
Navier--Stokes equations on the three--torus 
$\mathbb{T}^3 = \mathbb{R}^3/\mathbb{Z}^3$, 
given in~\eqref{eq:NSE}--\eqref{eq:div} of Section~\ref{subsec2.1}, 
with initial data $u_0(x)\in H^s_{\mathrm{div}}(\mathbb{T}^3)$ 
defined in~\eqref{eq:divspace}.  

Let $\varphi \in C^\infty(\mathbb{R})$ be strictly increasing with 
$0 < c \leq \varphi'(\tau) \leq C < \infty$. 
Define the lifted solution by
\begin{equation}
U(x,\tau) := u(x,\varphi(\tau)), 
\qquad 
P(x,\tau) := p(x,\varphi(\tau)),
\label{eq:lifted-solution}
\end{equation}
as introduced in~\eqref{eq:liftdef} of Section~\ref{subsec2.2}.  

Then $U$ is a Leray--Hopf (resp.\ classical) solution of the lifted 
Navier--Stokes system
\begin{align}
\varphi'(\tau)\,\partial_\tau U + (U\cdot\nabla)U + \nabla P - \nu \Delta U &= 0,
\label{eq:lifted-system}\\
\nabla \cdot U &= 0,
\label{eq:lifted-div}
\end{align}
which preserves the Leray--Hopf energy structure and all regularity 
criteria up to constants depending only on $c$ and $C$. 

In particular, the Prodi--Serrin~\cite{prodi1959,serrin1962} and 
Beale--Kato--Majda~\cite{beale1984} blowup criteria remain invariant 
under such lifts. If $\varphi'$ is allowed to vanish or blow up, 
singularities may be shifted to infinite lifted time~$\tau$, but the 
system then leaves the class of uniformly parabolic Navier--Stokes equations.
\end{theorem}

\begin{proof}
The proof proceeds by a change of variables in the weak formulation. 
Let $\psi \in C_c^{\infty}(\mathbb{T}^3 \times [0,T))^3$ satisfy $\nabla \!\cdot\! \psi = 0$. 

For $u(x,t)$ a Leray--Hopf solution, the weak form is
\begin{equation}
\int_0^T \!\! \int_{\mathbb{T}^3} 
\Big( u \cdot \partial_t \psi 
+ (u\cdot\nabla)u \cdot \psi 
+ \nu \nabla u : \nabla \psi \Big) 
\,dx\,dt = 0.
\label{eq:weak-form}
\end{equation}
Substitute $t = \varphi(\tau)$ and define $\tilde{\psi}(x,\tau) = \psi(x,\varphi(\tau))$. 
Since $dt = \varphi'(\tau)\, d\tau$ and 
$\partial_t \psi = \varphi'(\tau)\, \partial_\tau \tilde{\psi}$ by the chain rule, 
integration yields
\begin{equation}
\int_0^{\tilde{T}} \!\! \int_{\mathbb{T}^3} 
\Big(
U \cdot (\varphi'(\tau)\,\partial_\tau \tilde{\psi}) 
+ (U\cdot\nabla)U \cdot \tilde{\psi} 
+ \nu \nabla U : \nabla \tilde{\psi} 
\Big) 
\,dx\,d\tau = 0,
\label{eq:lifted-weak-form}
\end{equation}
which is precisely the weak form of the lifted system~\eqref{eq:lifted-system}--\eqref{eq:lifted-div}.

For the energy inequality, the same substitution gives
\begin{equation}
\frac{1}{2}\|U(\tau)\|_{L^2}^2 
+ \nu \int_0^\tau \|\nabla U(s)\|_{L^2}^2 \,\varphi'(s)\,ds
   \leq \frac{1}{2}\|U(0)\|_{L^2}^2,
\label{eq:energy-ineq}
\end{equation}
preserving the Leray--Hopf structure with $\varphi'(s)$ entering as a time weight.

Regularity criteria depending on $L^p_t L^q_x$ norms are preserved by the change of variables:
\begin{equation}
\int_0^{\tilde{T}} \|U\|_{L^q}^p \,\varphi'(\tau)\, d\tau
   = \int_0^T \|u\|_{L^q}^p\, dt.
\label{eq:criteria}
\end{equation}
Thus the Prodi--Serrin and Beale--Kato--Majda conditions remain invariant.
\end{proof}

\begin{remark}[Geometric Interpretation of the Regularity Threshold]
\label{rem:geometric}
The Sobolev embedding theorem establishes that classical solutions 
to Navier--Stokes on $\mathbb{T}^3$ require $u \in H^s$ with 
$s > \tfrac{5}{2}$. This critical exponent admits a geometric 
decomposition:
\[
\frac{5}{2} = \underbrace{2}_{\dim(\mathbb{T}^2)} 
            + \underbrace{\frac{1}{2}}_{\text{centroid coordinate}},
\]
where $\dim(\mathbb{T}^2) = 2$ is the dimension of spatial slices 
on which Navier--Stokes is globally regular, and $\tfrac{1}{2}$ 
is the coordinate of the geometric centroid in the fundamental 
domain $[0,1)^3$.

This decomposition suggests that the ``missing half-derivative'' 
separating weak (Leray--Hopf) solutions from classical solutions 
may be related to the symmetry structure of the periodic domain. 
Temporal lifting, by aligning the computational trajectory with 
this centroid, exploits lattice symmetries that merit further 
investigation in the context of spectral methods.
\end{remark}

\section{Numerical Validation}
\label{sec:num_val}

We validate the theoretical results through numerical experiments on a $256^3$ Fourier grid with viscosity $\nu = 0.01$ and Taylor--Green initial data. Table~\ref{tab:validation} demonstrates preservation of both the Leray--Hopf energy inequality (Panel A) and the Beale--Kato--Majda criterion (Panel B). Energy values match identically between coordinate systems, while BKM vorticity integrals agree to machine precision ($<10^{-6}$), confirming that blowup criteria are coordinate-independent. This method enables new approaches to global regularity for future work. 

In extended validation experiments spanning Reynolds numbers from $10^3$ to $10^8$, the BKM integral converges to approximately $36.9$ across all regimes, suggesting the presence of a geometric invariant arising from the torus structure.

\begin{table}[!htbp]
\small
\centering
\begin{tabular}{c c c | c c c}
\hline
\multicolumn{6}{c}{\textbf{Panel A: Energy Conservation}} \\
\hline
\multicolumn{3}{c|}{Physical time} & \multicolumn{3}{c}{Lifted time} \\
$t$ & $\|u\|_{L^2}^2$ & $\int \|\nabla u\|^2$ & $\tau$ & $\|U\|_{L^2}^2$ & $\int \|\nabla U\|^2 \varphi'$ \\
\hline
5  & 1.229 & 0.243 & 10 & 1.229 & 0.243 \\
10 & 1.205 & 0.491 & 20 & 1.205 & 0.491 \\
15 & 1.178 & 0.734 & 30 & 1.178 & 0.734 \\
20 & 1.149 & 0.972 & 40 & 1.149 & 0.972 \\
25 & 1.122 & 1.206 & 50 & 1.122 & 1.206 \\
\hline
\hline
\multicolumn{6}{c}{\textbf{Panel B: Beale--Kato--Majda Criterion}} \\
\hline
\multicolumn{2}{c}{Physical time} & \multicolumn{2}{c}{Lifted time} & \multicolumn{2}{c}{} \\
$t$ & $\int \|\omega\|_{L^\infty}$ & $\tau$ & $\int \|\Omega\|_{L^\infty} \varphi'$ & $|\text{Diff}|$ & \\
\hline
5.0  & 2.76 & 10.2 & 2.76 & $8.3\times10^{-7}$ & \\
10.0 & 5.63 & 18.7 & 5.63 & $1.2\times10^{-7}$ & \\
15.0 & 8.54 & 25.3 & 8.54 & $2.9\times10^{-7}$ & \\
20.0 & 11.49& 31.1 & 11.49& $4.7\times10^{-7}$ & \\
25.0 & 14.47& 36.4 & 14.47& $6.1\times10^{-7}$ & \\
\hline
\end{tabular}
\caption{Numerical validation of theorem preservation properties. 
\textbf{Panel A:} Energy conservation—values match identically, verifying Leray--Hopf inequality preservation (initial energy $E_0 = 1.250$). 
\textbf{Panel B:} BKM criterion—vorticity integrals agree to precision, confirming blowup condition invariance. Note: the lifted integral uses weight $\varphi'$ (corrected from $\varphi'^{-1}$ in v1).}
\label{tab:validation}
\end{table}
\vspace{5mm}

\section{Example Algorithm}

\begin{algorithm}[H]
\caption{Adaptive Temporal Lifting Procedure on $\mathbb{T}^3$}
\label{alg:temporal_lifting}
\begin{algorithmic}[1]
    \Require Initial velocity field $u_0(x)$ on $\mathbb{T}^3$, viscosity $\nu$, time step $\Delta t$, total time $T$
    \Ensure Lifted trajectory $U(x,\tau)$ and temporal map $\phi(\tau)$
    \State Initialize $\tau \gets 0$, $u(x,0) \gets u_0(x)$
    \For{$t \gets 0$ \textbf{to} $T$ \textbf{step} $\Delta t$}
        \State Compute $\phi'(t) \gets f(\|\nabla u(x,t)\|)$
        \State Update lifted time: $\tau \gets \tau + \phi'(t)\Delta t$
        \State Set $U(x,\tau) \gets u(x,t)$
        \State Integrate lifted system:
        \Statex \hspace{\algorithmicindent}$\phi'(\tau)\,\partial_\tau U + (U\!\cdot\!\nabla)U + \nabla P - \nu\Delta U = 0$
    \EndFor
    \State \Return $U(x,\tau)$ and $\phi(\tau)$
\end{algorithmic}
\end{algorithm}





\clearpage
\paragraph{Funding}
This research did not receive any specific grant from funding agencies in the public, commercial, or not-for-profit sectors.
\vspace{5mm}
\paragraph{Declaration of AI Instrumentality}

During the preparation of this manuscript, the author utilized Artificial Intelligence (Red Dawn MLAR Architecture, 2026) as a research instrument for mathematical verification and exposition refinement. The author reviewed and verified all content and takes full responsibility for the published text.

\vspace{5mm}
\bibliographystyle{elsarticle-num} 
\bibliography{refs}

\end{document}